\documentclass[a4paper]{article}

\usepackage{INTERSPEECH2019}
\usepackage{xcolor}
\usepackage[normalem]{ulem}

\title{Phrase-Level Class based Language Model for Mandarin Smart Speaker Query Recognition}
\name{Yiheng Huang$^{*}$\thanks{*The first two authors contribute equally.}, Liqiang He$^{*}$, Guangsen Wang, Lei Han and Dan Su}
\address{
  Tencent AI Lab
  }
\email{\{arnoldhuang,andylqhe,vincegswang,lxhan,dansu\}@tencent.com}
\begin{document}
\maketitle  
\newtheorem{definition}{Definition}
\newtheorem{theorem}{Theorem}
\newtheorem{proof}{Proof}

\begin{abstract}
The success of speech assistants requires precise recognition of a number of entities on particular contexts. A common solution is to train a class-based n-gram language model and then expand the classes into specific words or phrases. However, when the class has a huge list, e.g., more than 20 million songs, a fully expansion will cause memory explosion. Worse still, the list items in the class need to be updated frequently, which requires a dynamic model updating technique. In this work, we propose to train pruned language models for the word classes to replace the slots in the root n-gram. We further propose to use a novel technique, named Difference Language Model (DLM), to correct the bias from the pruned language models. Once the decoding graph is built, we only need to recalculate the DLM when the entities in word classes are updated. Results show that the proposed method consistently and significantly outperforms the conventional approaches on all datasets, esp. for large lists, which the conventional approaches cannot handle.
\end{abstract}

\section{Introduction}

Speech Assistants have recently gained its popularity in human's daily lives. Representative product in current market include the Amazon's Echos, Apple's Siri, Google Assistant, etc. In the according scenarios, the user queries usually contain explicit patterns such as ``Play \$SONG NAME\$'', ``I want to watch \$VIDEO NAME\$''Recognizing these names as the entities is challenging, and these semantic patterns are usually not well modelled in the general language model. 
On the other hand, new songs and videos arrive everyday, and it is expensive and impractical to frequently update the general language models daily to capture these new entities. Worse still, incorrectly recognized entities might raise strong negative impacts to the user experience. 
We refer this problem to as the ``hot'' word problem. One natural solution to this problem is~\cite{Aleksic2015}, in which a class based language model was first trained and was then converted to a weighted finite state transducer (WFST)~\cite{Mohri02} as the root-grammar. 
The entity names in the root-grammar were then replaced by the WFST generated from contact names. However, if the number of entities is large, directly replacing them into the root-grammar will lead to memory problems and is impossible to be compiled into a decoder graph.  
For example, if there are $n$ n-gram items containing the \$SONG-SLOT\$ in the root grammar and the size of the \$SONG\$ language model is $m$, replacing these n-grams will result in a language model with size $n \times m$. In our application, $n$ and $m$ are typically $10K$ and $20M$, respectively. Thus, the resulted language model will be of magnitude $200G$. This problem is referred to as the ``size exploding'' problem.

For Mandarin speech assistants, an additional challenge is that the word segmentation quality has a strong impact on the recognition performance, as the language models trained on the word-level are superior than the ones trained on the character-level. Generally speaking, the segmentation quality relies on the vocabulary. Thus, the vocabulary of a single general language model may not work well to special domains, such as song names, singer names and video names. This makes it necessary to use different vocabularies to model the queries related to these special domains. Combining multiple vocabularies is another major motivation of our work.

To this end, a phrase-level class-based language model is proposed in this work. The entities, for example, song names, video names and singer names are trained as n-gram models with different vocabularies which are suitable for their own tasks. 
In addition, these language models can be heavily pruned so that replacing the root-grammar will not blow up the language model. To address the ``size exploding'' problem, for decoding, we adopted the on-the-fly re-scoring in~\cite{Hori07} via the proposed \emph{difference language model} to get more accurate language model scores.
The root-grammar is a general language model with entities in classes to be replaced by their class names, which is trained from a large text corpus. 
For the sub-grammars such as song names, they can be trained very efficiently since the size is smaller than the root grammar in orders of magnitudes. By using the proposed method, once the decoder graph is built, we only need to recalculate the DLMs when the entities in word classes are updated while leaving the decoder graph stays unchanged. Because the calculations of DLMs are very fast, therefore, our system can be easily updated with the new items at minutely basis, thus addressing the ``hot'' word problem.

The proposed phrase-level class based language model is evaluated on our own dataset, namely, `speaker\_201812', `speaker\_201901'. They are voice search to the internal intelligent speaker song retrieval task, which consists of real user song queries containing both the wake-up words and the query utterances. 
Experiment results show that after interpolating with the general language models, significant gains can be obtained on the song retrieval task compared to the general language model with no loss of performance on the general recognition task such as reading speech, etc.
Our work bears some similarities with some recent works \cite{Hall2015, Aleksic2015-2, Vasserman2016, Velikovich18}. 
In \cite{Hall2015, Aleksic2015-2}, a biasing language model trained with the latest text inputs or queries is used to bias the scores of the general language model to better handle the recent trendy search queries. 
Furthermore, in \cite{Velikovich18}, semantic information was augmented to the language model by re-scoring the lattices from the first-pass decoding. With a powerful semantic model, significant gains can be obtained as shown in their work. The main differences of our work is that different classes of entities have their own vocabularies, and on the fly re-scoring is performed on the corresponding \textit{root-grammar} or \textit{sub-grammar} individually for every word. However, in \cite{Hall2015} they only re-score phrases in a pre-defined set. What is more, the decoder graph need to be built only once, when new entities are updated, only the DLMs correspond to sub grammars need to be rebuilt, that is a small size problem. e.g., in a magnitude of $100M$.

\label{MainSec1}

\section{ Difference Language Model}
In this section, we introduce the \emph{phrase-level class based language model} and the \emph{difference language models}. 
\label{MainSec2}

\subsection{Phrase-Level Class based Language Model}
\label{section1}

In the pioneer paper \cite{Brown1992-CB0}, the authors describe the class based language model on the word level. Instead, we propose phrase-level class based language model.


\subsubsection{Formulation}

Assume there is a vocabulary \textbf{V} and a set of classes \{${\mathcal{C}}_1,{\mathcal{C}}_2,...,{\mathcal{C}}_m$\}, and each word in \textbf{V} can only belongs to one of these classes. Given a word sequence w$_1$, w$_2$, ..., w$_n$ generated from the vocabulary, there exists a partition $\pi$ that separates a word sequence into a sequence of continuous phrases $W_1, W_2, ..., W_{\pi_(n)}$. We assume that the words appeared in one phrase are from the same class, and the words in any two neighbored phrases are from distinct classes. 
Denote the class label for phrase $W_i$ as $C_i$, where $C_i\in\{\mathcal{C}_1,\mathcal{C}_2,...,\mathcal{C}_m\}$. The probability of the word sequence P(w$_1$w$_2$...w$_n$) can be rewritten to P($W_1W_2...W_{\pi(n)}$) which can be decomposed as:
\begin{equation}
    P(W_1^{\pi(n)}) = P(W_1)P(W_2|W_1)\cdot\cdot\cdot P(W_{\pi(n)}|W_1^{\pi(n)-1})
    \label{equt1}
\end{equation}
Followed the definition of word class n-gram models as defined in \cite{Brown1992-CB0}, we give our new definition.
\begin{definition}
A language model is a \emph{phrase level class based language model}, if $P(W_k|W_1^{k-1}) = P(W_k|C_k)P(C_k|C_1^{k-1})$, $1\leq k \leq \pi(n)$, where $C_k$ is the class of phrase $W_k$.
\label{Def1}
\end{definition}
In our model, each phrase $W_i$ can only belong to exactly one class $\mathcal{C}_i$, so $P(W_i|\mathcal{C}_j)=0$ if $C_i\neq\mathcal{C}_j$.  The following theorem gives a sufficient condition for phrase level class based language model. 

\begin{theorem}
Assume there are $m$ classes and Eqs. (\ref{eqCond1}) and (\ref{eqCond2}) hold, then the language model is a \textit{phrase level class based language model} in Definition 1.
\begin{equation}
P(W_k|C_k, W_1^{k-1}) = P(W_k|C_k)
\label{eqCond1}
\end{equation}
\begin{equation}
P(C_k|W_1^{k-1}) = P(C_k|C_1^{k-1})
\label{eqCond2}
\end{equation}
\end{theorem}

\begin{proof}
We have
\begin{align}
P(W_k|W_1^{k-1}) & = \sum_{1 \leq i \leq m}P(W_k,W_k \in \mathcal{C}_i|W_1^{k-1} ) \nonumber\\ 
                 & = P(W_k,C_k|W_1^{k-1} ) \nonumber\\ 
                 & = Pr(W_k|C_k, W_1^{k-1} ) \cdot P(C_k | W_1^{k-1})
\end{align}
and thus if Eqs. (\ref{eqCond1}) and (\ref{eqCond2}) hold, obviously,  $P(W_k|W_1^{k-1}) = P(W_k|C_k)P(C_k|C_1^{k-1})$ , which completes the proof.
\end{proof}

The condition in Eq. (\ref{eqCond1}) is a mild assumption under our cases. For example, $C_k$ can be chosen as the song class SONG-SLOT, and $W_1^{k-1}$ can be chosen as commands such as `listen to' or `play'. The probabilities $P(\text{`Bad\ romance'}|\text{SONG-SLOT}, \text{`play'})$ and $P(\text{`Bad\  romance'}|\text{SONG-SLOT}, \text{`listen\ to'})$ are equal to each other, regardless of which command is used. 

\subsection{The Difference Language Model}
\label{section2}
In the pioneer work \cite{Hori07}, the authors introduce an on-the-fly re-scoring framework. A small language model is used to build the decoder graph, and a larger language model is used to re-score the LM scores in the decoding progress simultaneously. In our work, following their idea, a \emph{Difference Language Model} (DLM) is devised for re-scoring. 
\begin{definition}
A language model $A$ is a DLM of $B$ and $C$, if given arbitrary history $H$ and word $w$, Eq. (\ref{equtdef2}) holds as below:
\begin{equation}
    logP_A(w|H) = logP_B(w|H) - logP_C(w|H)
    \label{equtdef2}
\end{equation}
\label{Def2}
\end{definition}
\begin{theorem}
Denote $B$ and $C$ as two back-off n-gram language models with the same vocabulary \textbf{V}. The set of n-gram entries (without probabilities and back-off coefficients) in $C$, denoted as ${Ngram}_C$, is a subset of ${Ngram}_B$, i.e., ${Ngram}_C \subset {Ngram}_B$, if a back-off n-gram language model $A$ satisfies $Ngram_A=Ngram_B$. In addition, for each n-gram $\{H,w\} \in Ngram_B$,
\begin{equation}
\log P_A(w|H) = \log P_B(w|H) - \log P_C(w|H),
\label{equt5}
\end{equation}
\begin{equation}
    \alpha_A(H) = \alpha_B(H) - \alpha_C(H),
\label{equt6}
\end{equation}
where $\alpha(H)$ is the back-off parameter of history $H$, and $\alpha_C(H)$ is zero if $H$ is not contained in $Ngram_C$. Then, the language model $A$ is a DLM of $B$ and $C$.
\label{theorem2}
\end{theorem}
\begin{proof}
We prove the theorem by deduction. For uni-grams $\forall w \in \textbf{V}$, Eq. (\ref{equtdef2}) holds trivially. Assume Eq. (\ref{equtdef2}) holds for any sequence with length no longer than $l$, and $\{H,w\}$ is a sequence of length of $l+1$. Supposing $H=\{h_1,\cdots,h_l\}$, denote $H'=\{h_2,\cdots,h_l\}$ as the suffix of $H$. 
If $\{H,w\} \in Ngram_B$, Eq. (\ref{equtdef2}) holds by definition. If $\{H,w\} \notin Ngram_B$, we have $\{H,w\} \notin Ngram_C$ since $Ngram_C \subset Ngram_B$. According to the definition of the back-off language model, $\log P_B(w|H) - \log P_C(w|H) = \log P_B(w|H') + \alpha_B(H) - \log P_C(w|H') - \alpha_C(H)$. Since $\log P_B(w|H') - \log P_C(w|H') = \log P_A(w|H')$ by the assumption and $\alpha_A(H) = \alpha_B(H) - \alpha_C(H)$ (when $H \notin Ngram_B$,  $\alpha_A(H) = \alpha_B(H) = \alpha_C(H) = 0$), thus $\log P_A(w|H) = \alpha_A(H)+\log P_A(w|H') = \log P_B(w|H) - \log P_C(w|H)$. So, Eq. (\ref{equtdef2}) holds for $\{H,w\}$ with length $l+1$, and we complete the proof by deduction. 
\end{proof}
From Theorem \ref{theorem2}, we can directly compute the DLM.

\section{Dynamic Decoding}
\label{MainSec3}
In this section, we briefly introduce the DLMs, decoder graph, and the on-the-fly re-scoring progress.
\subsection{The Language Model}
The proposed phrase-level class based language model consists of a root grammar model and a sub-grammar model:
\begin{itemize}
\item the root grammar is the language model characterized by $P(C_k|C_1^{k-1})$;
\item the sub-grammar is responsible for modeling the emitting probabilities $P(W_k|C_k)$ of entities $W_k$ given the class $C_k$.
\end{itemize}

We refer the classes to as ``slots'' denoted by SLOT-NAME.  The root grammar is trained by SRILM \cite{Stockle02} from various corpora containing training sentences such as `play SONG-SLOT', with the concrete entities substituted by their class names. 
The corpora are from manual transcripts of online search queries. 
Each word in these sentences is treated as an individual class. 
Each word in the root grammar vocabulary is concatenated with a prefix `class\_', e.g., word `play' is replaced by `class\_play'. The  root grammar vocabulary contains 213893 ordinary Chinese words plus 3 extra words correspond to the sub classes (i.e., song, singer and video). Furthermore, the root grammar can be interpolated with the general language model trained from other sources such as news and conversations, etc. Finally, the model is pruned and the corresponding DLM is built using the method in section \ref{MainSec2}. 

We build three sub-grammar models based on their distinct databases. Once the sub-grammars are built, they are directly pruned to arbitrary smaller n-grams (at the extreme cases, that is the uni-gram), and the DLMs are built accordingly. In our databases, more than 20 million songs and 1 million videos are available, and there is a list of more than 200 thousands of singers. The vocabulary sizes corresponding to the classes SONG, VIDEO and SINGER are 50687, 43797 and 17717, respectively. The names of the words are denoted by appending the class name as a prefix, e.g., the word $w$ in class SONG-SLOT is denoted by SONG-SLOT\_$w$.

Finally, the DLMs of both the root grammar and sub-grammars can be obtained according to Theorem \ref{theorem2}. These n-grams are then converted to a tree structure as described in~\cite{Soltau09}, where the language model state is encoded as a 64-bit value. The state contains the information of the
depth of the n-gram tree and n-gram history. Based on the information, it is straightforward to look up the n-gram scores for the incoming words in the n-gram tables.

\label{section3.1}

\subsection{The Decoder Graph}
The language models are converted to WFST format using Kaldi~\cite{Kaldi} and  the replacements of sub-grammars are conducted using OpenFst \cite{OpenFst}. Some example WFSTs are shown in Figs.~\ref{fig:WFST_root}, \ref{fig:WFST_song} and \ref{fig:WFST3}.

\begin{figure}[t]
  \centering
  \includegraphics[width=\linewidth]{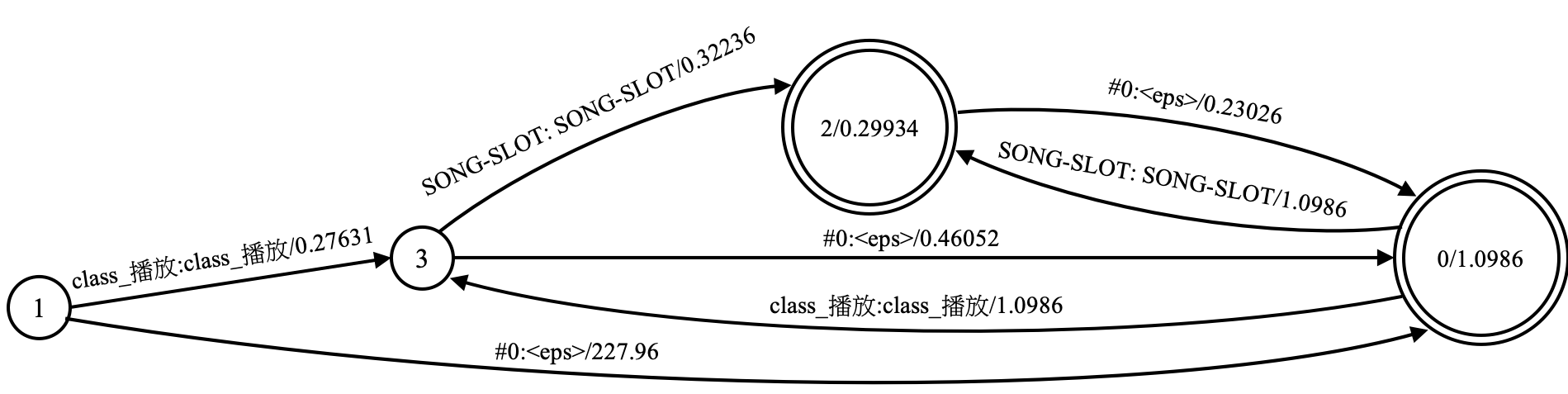}
  \caption{WFST of root grammar}
  \label{fig:WFST_root}
\end{figure}
\begin{figure}[t]
  \centering
  \includegraphics[width=\linewidth]{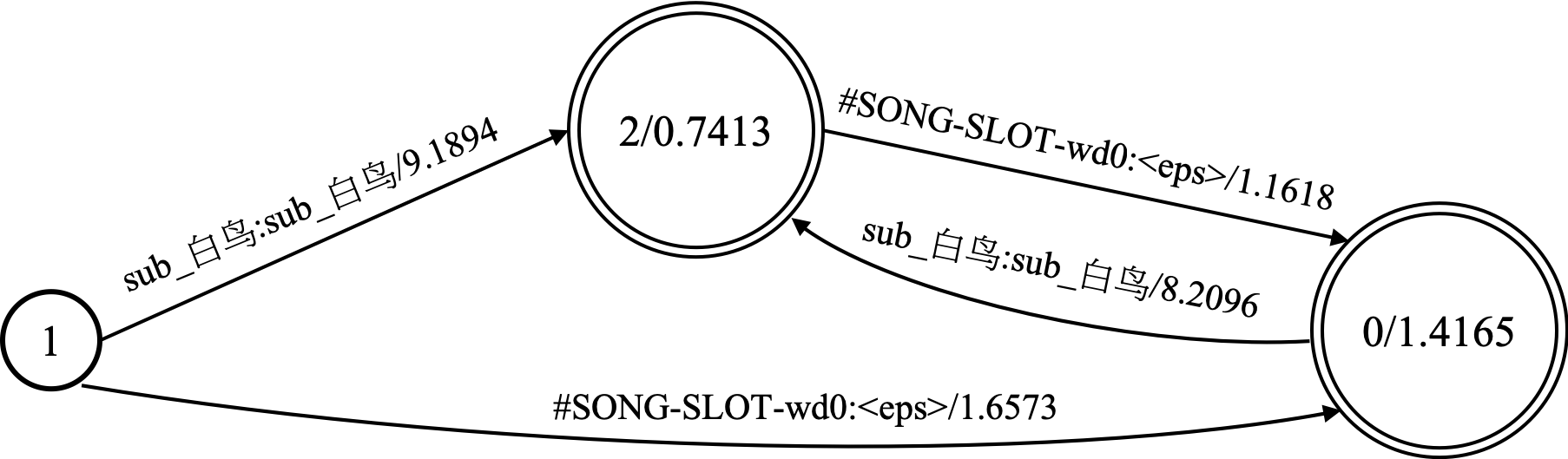}
  \caption{WFST of sub grammar}
  \label{fig:WFST_song}
\end{figure}
\begin{figure*}[t]
  \centering
  \includegraphics[width=\linewidth]{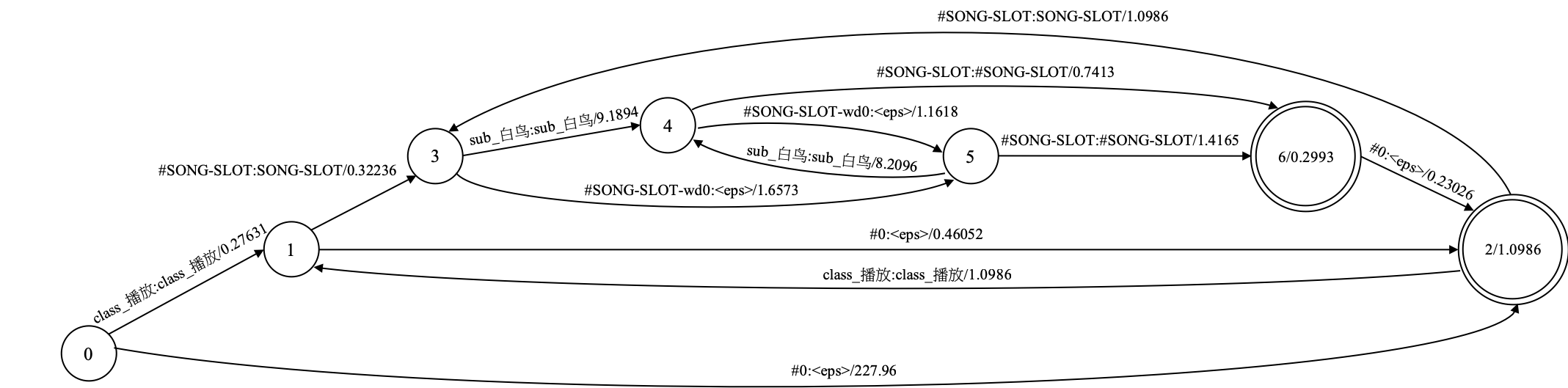}
  \caption{The replaced WFST}
  \label{fig:WFST3}
\end{figure*}
The root grammar WFST shown in Fig.~\ref{fig:WFST_root} is converted from a bi-gram trained from a single sentence indicating `Play SONG-SLOT' in Chinese. 
The bi-gram corresponding to the sub-grammar WFST in Fig.~\ref{fig:WFST_song} is trained from a single song name indicating `white bird' in Chinese. The symbol `\#0' is a disambiguation symbol in the root grammar WFST, and `\#SONG-SLOT-wd0' is a disambiguation symbol in sub-grammar WFST. 
Furthermore, if `SONG-SLOT' is the input over an arc, it is replaced by `\#SONG-SLOT' to be a disambiguation symbol, which creates the final HCLG.fst \cite{Mohri02}.
\label{Section3.2}

\subsection{Dynamic Decoding}
Our decoder structure is similar to what is applied in~\cite{Saon05}. During decoding, an on-the-fly re-scoring strategy similar to \cite{Hori07} is performed. Since we are using the class-based DLMs, some key information that we keep in the decoding progress include:
\begin{itemize}
    \setlength{\itemsep}{0pt}
    \setlength{\parsep}{0pt}
    \setlength{\parskip}{0pt}
    \item states corresponding to the first pass WFST;
    \item the language model ID being searched during the on-the-fly re-scoring;
    \item states of the current language model;
    \item a back-up state to backup the original state in the \emph{root} DLM.
\end{itemize}

We provide more details when the decoding progress switches between the root grammar and sub-grammars. When token enters the sub-grammar, (e.g, an output label `SONG-SLOT' is observed), the decoder switches to the corresponding sub-grammar, initializes the DLM state, backs-up the state in the root grammar, and then precedes the decoding progress. 
On ther other side, when the decoder leaves the sub-grammars, (e.g., an output label `\#SONG-SLOT' is observed), the DLM score corresponds to the end-of-sentence is added, and then the decoder switches to the backed up state of the root grammar and continues the decoding.
A quadruple is used to record all these information. 
During decoding, when a word-emitting arc is traversed, we look up the extension of the previous state in the corresponding DLM, and if the word is found, we combine this new DLM state with the current decoder state, DLM ID and the backup state to form a new quadruple.

\label{section3.3}

\section{Experiments}
In this section, we first conduct extensive experiments by comparing with the conventional approaches, such as \cite{Aleksic2015}, with small model size. 
For models with large size, we compare the performance of our method with common language model that is trained on a very large number of common corpora.

\subsection{Experimental setup}
The TDNN-LSTMP-LFMMI~\cite{Povey15} acoustic model used in the study was trained with six thousands of hours of general purpose Mandarin data consisting of mostly read speech. 
The model is then fine-tuned with two thousands of hours of internal data from the Tencent TingTing. 
The neural network has 7 TDNN layers interleaved with 4 LSTMP layers. 
The output target contains 9782 bi-phone senones obtained from a Mandarin syllable lexicon. 
For evaluation, two test sets are used: `voice\_speaker\_201812' contains 735 recent user queries to the Tencent TingTing intelligent speaker, while `voice\_speaker\_201901' contains 1000 similar queries. 
To demonstrate that our method does not harm the general purpose recognition task, another testing set `AI\_Lab\_test\_600' containing 600 regular read speech utterances is used. 
The language model is built as described in section \ref{section3.1}. 
All the n-gram language models are trained using SRILM \cite{Stockle02}.

\label{section4.1}

\subsection{Performance Consistency}
To demonstrate that our method does not loss any precision, we should use the original sub-grammars without pruning to build the baseline results. 
However, since the original sub-grammars are too big to be used directly, (e.g, very large sub-grammars trained over 20 million songs). 
These sub-grammars are first pruned moderately, following the similar pipeline in \cite{Aleksic2015}, to build the baseline system. 
To build our on-the-fly system, the slightly pruned sub-grammars are used as the formal grammars. 
Then, these grammars are further pruned to be inserted into the root grammar, and we then use the method described in section \ref{MainSec2} to build the DLMs accordingly. Finally, the algorithm proposed in section \ref{section3.3} is used to perform the on-the-fly re-scoring. 
\linespread{0.8}
\begin{table}[th]
   \caption{Consistency results}
   \label{tab:consistency}
    \centering
    \setlength{\tabcolsep}{1.0mm}
    \begin{tabular}{|c|c|c|c|}
      \hline
      \hline
      & \textbf{Baseline\cite{Aleksic2015}} & \textbf{rescore(sub)} & \textbf{rescore(root+sub)} \\
      \hline
       \textbf{spk\_201812} & $4.3$  &  $4.3$  &  $4.2$     \\
       \textbf{spk\_201901} & $9.6$ & $9.4$ & $9.2$  \\
    \hline
    \hline
    \end{tabular}
\end{table}

Table~\ref{tab:consistency} reports the results showing the performance consistency between our method and the baseline method. 
The last two columns of Table~\ref{tab:consistency} stand for the on-the-fly re-scoring results. Column 2 indicates the results where only the sub-grammars are pruned, while column 3 shows the results when all grammars are pruned. Results show that our method achieves the same accuracy as that as baseline on the test set `spearker\_201812' and even better results in `speaker\_201901'. 
A reasonable explanation for this is that the on-the-fly approach has a much smaller decoder graph compared with the baseline, so it might be more straightforward to reach the correct LM score, and the unreliable decoding path can be cut more efficiently, leading to better results.

\subsection{In-Domain and Out-Domain Testing}
The experimental results in Table~\ref{tab:performance} show the performance after interpolating with common language model. 
The first column records the result of decoding with only root grammar, the second column records the result of decoding with common language model, while the last column stands for the result after interpolating. 
As clearly observed, we achieve significant performance improvements, approaching a rate of $10\%-15\%$ relatively, on in-domain sets, while the performance on out-domain sets decays slightly.
\begin{table}[th]
   \caption{On-line voice search results}
   \label{tab:performance}
    \centering
    \begin{tabular}{|c|c|c|c|}
      \hline
      \hline
      & \textbf{root} & \textbf{common} & \textbf{root + common} \\
      \hline
       \textbf{speaker\_201812} & $4.3$  &  $12.9$  &  $\mathbf{3.7}$     \\
       \textbf{speaker\_201901} & $9.4$ & $17.1$ & $\mathbf{8.6}$  \\
       \textbf{ai\_lab\_test600} & $30.3$ & $11.1$ & $11.4$  \\
    \hline
    \hline
    \end{tabular}
\end{table}
To compare the decoding time cost, \emph{rtfs} on these test sets are reported. The on-the-fly method has similar decoding speed on the in-domain test sets compared to the common language model. On the out-domain test sets, the common language model shows faster decoding speed mainly because of paths with less confusion.
\linespread{0.4}
\begin{table}[th]
   \caption{Real time factor}
   \label{tab:rtf}
    \centering
    \begin{tabular}{|c|c|c|c|}
      \hline
      \hline
      & \textbf{root} & \textbf{common} & \textbf{root + common} \\
      \hline
       \textbf{speaker\_201812} & $0.1$  &  $0.1$  &  $0.11$     \\
       \textbf{speaker\_201901} & $0.11$ & $0.11$ & $0.12$  \\
       \textbf{ai\_lab\_test600} & $0.21$ & $0.17$ & $0.23$  \\
    \hline
    \hline
    \end{tabular}
\end{table}
\subsection{Size Reduction and Efficient Updating}
As mentioned in the above sections, when the original sub-grammars without pruning are inserted directly into the root grammar, the resulting G.fst becomes extremely large, i.e., with a size of $46G$. 
However, after pruning, the size of G.fst is reduced to $938M$. 
This implies that the proposed method is able to add considerable large items into the system by pruning the sub-grammars. 
The remaining task is to recalculate the DLMs of new sub-grammars, leaving the decoder graph unchanged. 
All of the computations can be done in a few minutes, and thus we can easily catch up with newly arrived items. 

\section{Conclusions and Future Work}
In this paper, we have proposed a phrase-level class based language model and an on-the-fly re-scoring method to address the `hot' word problem. 
Using this framework, we can catch up with new entities precisely and efficiently. 
In addition, experimental results show that we achieve significant performance improvements compared to the common language model. 
In our future work, we are interested to remove the replacement procedure to make the pipeline simpler.


\clearpage

\bibliographystyle{IEEEtran}


\end{document}